\documentclass[twocolumn,secnumarabic,amssymb, nobibnotes, aps, prx]{revtex4-2}

\usepackage{amsmath, amssymb, amsthm} 
\usepackage{enumitem} 
\usepackage{graphicx} 
\usepackage{bm} 
\usepackage{color}
\usepackage{float}
\restylefloat{table}
\newtheorem{proposition}{Proposition}
\newtheorem{algorithm}{Algorithm}
\DeclareMathOperator{\E}{\mathbb{E}}

\DeclareMathOperator{\Cov}{cov}

\begin{document}

\newcommand{\ML}[1]{\textcolor{red}{[ML:#1]}}

\title{Rayleigh-Gauss-Newton optimization \\ with enhanced sampling for variational Monte Carlo}%

\author{Robert J. Webber}%
\affiliation{Courant Institute of Mathematical Sciences, New York University, New York 10012, USA}

\author{Michael Lindsey}%
\affiliation{Courant Institute of Mathematical Sciences, New York University, New York 10012, USA}

\begin{abstract}
Variational Monte Carlo (VMC) is an approach for computing ground-state wavefunctions that has recently become more powerful due to the introduction of neural network-based wavefunction parametrizations.
However, efficiently training neural wavefunctions to converge to an energy minimum remains a difficult problem. 
In this work,
we analyze optimization and sampling methods used in VMC and introduce alterations to improve their performance.
First, based on theoretical convergence analysis in a noiseless setting, we motivate a new optimizer that we call the Rayleigh-Gauss-Newton method, which can improve upon gradient descent and natural gradient descent to achieve superlinear convergence at no more than twice the computational cost.
Second, in order to realize this favorable comparison in the presence of stochastic noise, we analyze the effect of sampling error on VMC parameter updates and experimentally demonstrate that it can be reduced by the parallel tempering method.
In particular, we demonstrate that RGN can be made robust to energy spikes that occur when the sampler moves between metastable regions of configuration space.
Finally, putting theory into practice, we apply our enhanced optimization and sampling methods to the transverse-field Ising and XXZ models
on large lattices, yielding ground-state energy estimates with remarkably high accuracy after just 200 parameter updates.
\end{abstract}

\maketitle

\section{Introduction}
Computing the ground-state wavefunction of a many-body Hamiltonian operator is a demanding task, 
requiring the solution of an eigenvalue problem whose cost grows exponentially with system size
in traditional numerical approaches.
Variational Monte Carlo (VMC, \cite{gubernatis2016quantum,becca2017quantum}) is an alternative strategy that avoids this curse of dimensionality by using stochastic optimization to find the best wavefunction within a tractable function class.

VMC has recently seen rapid and encouraging development due to the incorporation of insights from the machine learning community.
In 2017, Carleo and Troyer \cite{carleo2017solving} applied VMC with a two-layer neural network ansatz to accurately represent the ground-state wavefunction
of quantum spin systems with as many as $100$ spins. 
Since then, there has been major progress in extending neural network-based VMC to the setting of electronic structure, including the development of the neural network backflow ansatz for second-quantized lattice problems \cite{luo2019backflow}, as well as of FermiNet \cite{pfau2020ab} and PauliNet \cite{hermann2020deep} for quantum chemistry problems in first quantization.
These new approaches have been extended to systems as large as bicyclobutane
($\mathrm{C}_4\mathrm{H}_6$), which has $30$ interacting electrons~\cite{pfau2020ab,spencer2020better}.

VMC is highly flexible, since it extends without significant 
modification to systems of arbitrary spatial dimension. 
However, the price paid for this flexibility is a difficult optimization problem that relies on Monte Carlo sampling.
Efficiently solving this optimization problem has proven challenging.
Recent works \cite{luo2019backflow,hermann2020deep,spencer2020better,sharir2020deep,yang2020deep,yang2020scalable} raise concerns about the speed and stability of wavefunction training and report that 
VMC can suffer from long training times \cite{pfau2020ab,spencer2020better}, lose stability \cite{yang2020scalable}, or converge to unreasonable solutions \cite{park2020geometry}.
Thus there is motivation for the development of faster and more stable optimization and sampling solutions.

Our goal is to apply numerical and probabilistic analysis to evaluate and improve upon the optimization and sampling strategies in VMC. 
To this end, we first provide a unified perspective on several major VMC optimizers, namely gradient descent, quantum
natural gradient descent (also known as stochastic reconfiguration), and the linear method.
Reviewing these methods in a unified way clarifies a path toward 
improvement.
Specifically, we introduce a new \emph{Rayleigh-Gauss-Newton} (RGN) method and prove RGN achieves superlinear convergence as the wavefunction approaches the ground state.

Next we analyze the Markov chain Monte Carlo (MCMC) sampling used in VMC.
We establish a quantitative extension of the zero-variance principle \cite{gubernatis2016quantum,becca2017quantum} of VMC that we call the vanishing-variance principle.
This principle guarantees that the energy estimates converge to the true energy as the wavefunction nears an eigenstate.
However, away from an eigenstate, the accuracy of the energy estimates is \emph{not} guaranteed.
The energy estimates can have a high variance and can even exhibit energy spikes (see Figure \ref{fig:tempering}).
To stabilize these energy estimates, variance reduction strategies are needed.
Using a standard MCMC sampler as in \cite{carleo2017solving}, the wavefunction is slow to recover from the energy spikes ($\sim 10^3$ iterations);
however, using the parallel tempering MCMC method
\cite{swendsen1986replica},
the recovery period is much quicker ($\sim 10^2$ iterations).
Variance reduction strategies such as parallel tempering can be essential for realizing the full potential of VMC in large-scale applications.

Lastly, by using the Rayleigh-Gauss-Newton method along with parallel tempering, we obtain highly accurate variational estimates for the ground-state energies of transverse-field Ising and XXZ models with as many as $400$ spins.
Compared to past benchmark results obtained using natural gradient descent \cite{carleo2017solving}, we obtain the same or higher accuracy in fewer iterations.
Since RGN is only slightly more expensive than natural gradient descent, by less than a factor of two in our tests, we conclude that RGN can improve the overall efficiency of VMC.

The rest of the paper is organized as follows.
Section \ref{sec:vmc} gives an overview of variational Monte Carlo,
Section \ref{sec:optimization} analyzes optimization methods,
Section \ref{sec:sampling} analyzes sampling methods,
Section \ref{sec:results} presents numerical experiments,
and Section \ref{sec:conclusion} concludes.

Throughout the paper, $\Re z$ denotes the real part of a complex number $z$.
$\bm{v}^T$, $\overline{\bm{v}}$, and $\bm{v}^{\ast}$ denote the transpose, complex conjugate, and conjugate transpose of a vector $\bm{v}$,
and similar conventions are adopted for matrices.
We use single bars $\left| \cdot \right|$ for the Euclidean norm of a scalar, vector, or matrix
and $\left\lVert \cdot \right\rVert_2$ for the spectral norm of a matrix.
Lastly, we consider a 
finite- or infinite-dimensional
Hilbert space
of unnormalized wavefunctions $\psi$
and use $\left<\cdot, \cdot\right>$ and $\left\lVert \cdot \right\rVert$ to denote the associated inner product and norm.

\section{Overview of VMC}{\label{sec:vmc}}

The main goal of variational Monte Carlo (VMC) is the identification of the ground-state energy and wavefunction of the Hamiltonian operator $\mathcal{H}$
for a quantum many-body system.
We denote the ground-state energy and wavefunction using $\lambda_0$ and $\psi_0$, respectively.
In addition to solving the eigenvalue equation $\mathcal{H} \psi_0 = \lambda_0 \psi_0$,
these admit a variational characterization in terms of the energy functional
\begin{equation}
\mathcal{E}\left[\psi\right] 
= \frac{\left<\psi, \mathcal{H} \psi\right>}
{\left<\psi, \psi\right>}.
\end{equation}
The ground-state energy $\lambda_0$ is the minimum value of $\mathcal{E}$,
and the ground-state wavefunction $\psi_0$ is the minimizer, which we assume to be unique up to an arbitrary multiplicative constant.

Identifying $\lambda_0$ and $\psi_0$ becomes
difficult when the Hilbert space associated with $\mathcal{H}$ is high-dimensional
or infinite-dimensional.
For example, in the Heisenberg model for spin-$1 \slash 2$ particles on a graph \cite{Sachdev_2009}, $\mathcal{H}$ is the operator
\begin{equation}
	\mathcal{H} = \sum_{i \sim j} [J_x \sigma_i^x \sigma_j^x 
	+ J_y \sigma_i^y \sigma_j^y
	+ J_z \sigma_i^z \sigma_j^z]
	+ h \sum_i \sigma_i^x,
\end{equation}
where $\sigma_i^x$, $\sigma_i^y$, and $\sigma_i^z$ are Pauli operators
for the $i$-th spin,
$i \sim j$ signifies that $i$ and $j$ are neighboring spins, and $J_x$, $J_y$, $J_z$, and $h$ are real-valued parameters.
In the case, e.g., of a $10 \times 10$ square lattice,
the ground-state wavefunction can be viewed as a vector of length $2^{100}$, 
which is far too large to store in memory,
much less calculate with any conventional eigensolver, direct or iterative.

VMC must approximate this high-dimensional eigenvector
using a tractable parametrization
$\psi = \psi_{\bm{\theta}}$, 
where $\bm{\theta}$ is a vector of real- or complex-valued parameters.
VMC uses an iterative approach for updating the $\bm{\theta}$ parameters, with the goal of minimizing the energy within the parametric class.
VMC iterates over the following three steps.
\begin{enumerate}[itemsep=2pt,parsep=2pt]
	\item Draw random samples from the wavefunction density
	$\rho_{\bm{\theta}} = \left|\psi_{\bm{\theta}}\right|^2 \slash \left<\psi_{\bm{\theta}}, \psi_{\bm{\theta}}\right>$.
	\item Use the random samples to estimate the 
	energy $\mathcal{E}\left[\psi_{\bm{\theta}}\right]$,
	the energy gradient $\nabla_{\bm{\theta}} \mathcal{E}\left[\psi_{\bm{\theta}}\right]$,
	and possibly other quantities needed for the optimization.
	\item Update the $\bm{\theta}$ parameters to reduce the energy.
\end{enumerate}

In VMC,
we ideally find that
the estimated energies fall quickly in the first iterations
and decrease more slowly at subsequent iterations,
yielding increasingly accurate estimates of $\lambda_0$, as shown in Figure \ref{fig:mc_error}.

\begin{figure}[h!]
    \centering
    \includegraphics[scale=.35]{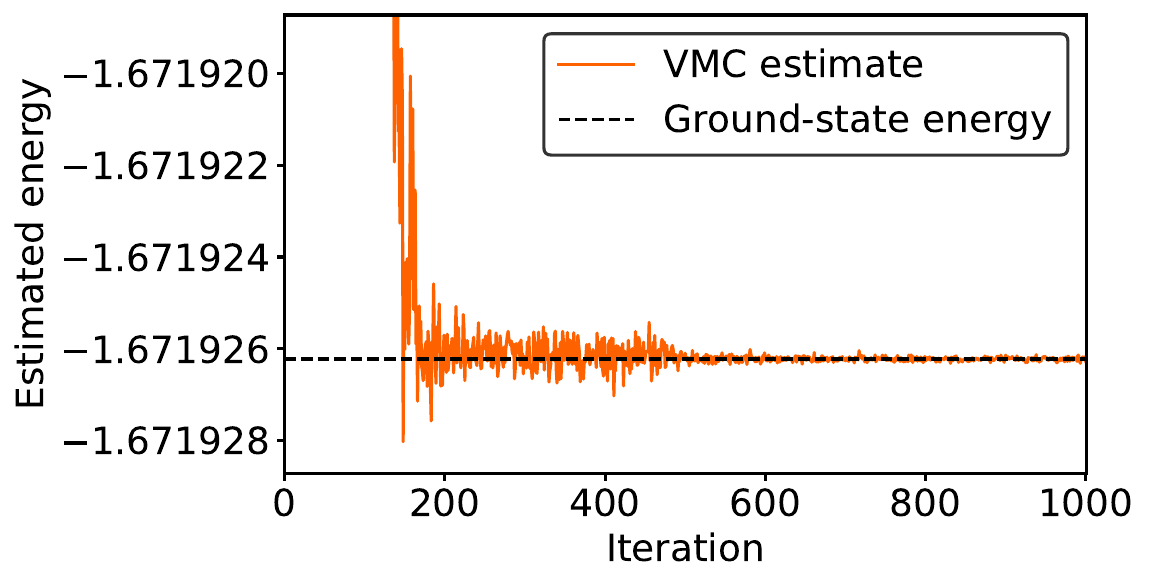}
    \caption{VMC ground-state energy estimates
    for a $200 \times 1$ Ising model with a transverse magnetic field ($h = 1.5$).
    Computational details are provided in Section \ref{sec:results}.}
    \label{fig:mc_error}
\end{figure}

Additionally, as seen in Figure \ref{fig:mc_error},
there is a \emph{vanishing-variance principle} by which the energy estimator's variance converges to zero
as the wavefunction approaches the ground state of $\mathcal{H}$ 
(see Proposition \ref{prop:zero_variance}).
Because of this principle, 
reductions in the energy mean and reductions in the energy variance both indicate that the wavefunction is approaching the ground state.
The vanishing-variance principle is essential in applications, since it enables VMC to provide accurate energy estimates even though the variance at the early stages of the optimization would appear to render such high accuracy impossible.

\section{Optimization approaches}{\label{sec:optimization}}

In this section, we obtain formulas for the energy gradient and Hessian, 
use these formulas to motivate optimization methods for VMC,
and lastly derive theoretical convergence rates for VMC optimizers.
Throughout the section, we assume that optimization methods are applied exactly without any Monte Carlo sampling.

\subsection{The energy gradient and Hessian}
\label{sub:derivations}

To begin, we derive formulas for the energy gradient and Hessian with respect to the parameters. 
By adopting the convention of intermediate normalization \cite{szabados2016perturbation}, we obtain compact expressions for these quantities that differ from past presentations, e.g., \cite[ch. 9]{gubernatis2016quantum}.

We fix a vector of parameters $\bm{\theta}$ and consider a small parameter update $\bm{\theta} + \bm{\delta}$.
The resulting wavefunction, after
intermediate normalization, is written 
\begin{equation}
\widehat{\psi}_{\bm{\theta} + \bm{\delta}}
= \frac{\left<
	\psi_{\bm{\theta}}, 
	\psi_{\bm{\theta}}
	\right>}
{\left<\psi_{\bm{\theta}}, 
	\psi_{\bm{\theta} + \bm{\delta}}\right>} \psi_{\bm{\theta} + \bm{\delta}}.
\end{equation}
This intermediate-normalized wavefunction
is a scalar multiple of
the unnormalized wavefunction $\psi_{\bm{\theta} + \bm{\delta}}$
and hence has the same energy.
However, $\widehat{\psi}_{\bm{\theta} + \bm{\delta}}$
has been rescaled to fix the inner product with $\psi_{\bm{\theta}}$.

We assume that $\bm{\delta} \mapsto \widehat{\psi}_{\bm{\theta} + \bm{\delta}}$ is a locally analytic function of real or complex parameters and consider the second-order Taylor series expansion
\begin{equation}
\label{eq:taylor}
	\widehat{\psi}_{\bm{\theta} + \bm{\delta}} = 
	\widehat{\psi}
	+ \sum_i \bm{\delta}_i \widehat{\psi}_i
	+ \frac{1}{2} \sum_{ij} \bm{\delta}_i \bm{\delta}_j
	\widehat{\psi}_{ij}
	+ \mathcal{O}(\left|\bm{\delta}\right|^3),
\end{equation}
where $\widehat{\psi}$, $\widehat{\psi}_i$, and $\widehat{\psi}_{ij}$
denote the normalized wavefunction and its partial derivatives
\begin{equation}
    \widehat{\psi} = \widehat{\psi}_{\bm{\theta}}, \qquad
    \widehat{\psi}_i = \partial_{\bm{\theta}_i} \widehat{\psi}_{\bm{\theta}},
    \qquad \widehat{\psi}_{ij} = \partial^2_{\bm{\theta}_i \bm{\theta}_j} \widehat{\psi}_{\bm{\theta}}.
\end{equation}

Manipulating \eqref{eq:taylor},
we then decompose the energy difference
$\mathcal{E}\big[\widehat{\psi}_{\bm{\theta} + \bm{\delta}}\big] - \mathcal{E}\big[\widehat{\psi}_{\bm{\theta}}\big]$ 
into the sum of gradient and Hessian terms
\begin{multline}
	\underbrace{\mathcal{E}\big[\widehat{\psi}_{\bm{\theta} + \bm{\delta}}\big]
	- \mathcal{E}\big[\widehat{\psi}_{\bm{\theta}}\big]}_{\text{energy difference}} \\
	= \underbrace{\bm{\delta}^{\ast} \bm{g} + \bm{g}^{\ast} \bm{\delta}}_{\text{gradient terms}}
	+ \underbrace{\bm{\delta}^{\ast} \bm{H} \bm{\delta}
		+ \Re (\bm{\delta}^T \,\overline{\bm{J}}\, \bm{\delta})}_{\text{Hessian terms}} + \mathcal{O}(\left|\bm{\delta}\right|^3).
\end{multline}
These gradient and Hessian terms are given explicitly by
\begin{align}
\label{eq:definitions_1}
	\bm{g}_i = \frac{\big<\,\widehat{\psi}_i,
		\widehat{\mathcal{H}}\,
		\widehat{\psi}\,\big>}
	{\big<\,\widehat{\psi}, \widehat{\psi}\,\big>},
	&&
	\bm{H}_{ij} = \frac{\big<\,\widehat{\psi}_i,
		\widehat{\mathcal{H}}\,
		\widehat{\psi}_j\,\big>}
	{\big<\,\widehat{\psi}, \widehat{\psi}\,\big>}, && \\
	\label{eq:definitions_2}
	\bm{J}_{ij} = \frac{\big<\,\widehat{\psi}_{ij},
		\widehat{\mathcal{H}}\,
		\widehat{\psi}\,\big>}
	{\big<\,\widehat{\psi}, \widehat{\psi}\,\bigr>}, &&
\end{align}
where $\widehat{\mathcal{H}} = \mathcal{H} - \mathcal{E}\big[\widehat{\psi}\big]$ is an energy-shifted version of the operator $\mathcal{H}$. 

Equations \eqref{eq:definitions_1} and \eqref{eq:definitions_2} offer transparent formulas for the energy gradient and Hessian.
In the case of real-valued parameters,
the energy gradient is $2 \bm{g}$,
and the energy Hessian is
$2 \bm{H} + 2 \bm{J}$.
In the case of complex-valued parameters (such as in the setting of \cite{carleo2017solving}),
the Wirtinger gradient \cite{wirtinger1927formalen, schreier_scharf_2010} of the energy is
$\big(\begin{smallmatrix}
\bm{g} \\
\overline{\bm{g}}
\end{smallmatrix}\big)$,
and the Wirtinger Hessian is 
$\big(\begin{smallmatrix}
\bm{H} & \bm{J} \\
\overline{\bm{J}} & \overline{\bm{H}}
\end{smallmatrix}\big)$.

The structure of the Hessian simplifies near the ground state, 
since $\bm{J} \rightarrow \bm{0}$
as the wavefunction approaches any eigenstate of $\mathcal{H}$.
\begin{proposition}{\label{prop:regularity}}
    The matrix $\bm{J}$ is bounded by
    \begin{equation}
		\left| \bm{J}_{ij} \right|
		\leq \frac{\big\lVert\, \widehat{\psi}_{ij}\, \,\big\rVert}
		{\big\lVert\, \widehat{\psi} \,\big\rVert} 
		\min_{\lambda \in \mathbb{R}}
		\frac{\big\lVert \left(\mathcal{H} - \lambda \right)
			\widehat{\psi} \,\big\rVert}
		{\big\lVert \,\widehat{\psi} \,\big\lVert}
    \end{equation}
    Therefore, $\bm{J} \rightarrow \bm{0}$ as $\min_{\lambda \in \mathbb{R}} 
    \big\lVert \left(\mathcal{H} - \lambda \right)
	\widehat{\psi} \,\big\rVert \slash \big\lVert\widehat{\psi} \,\big\rVert \rightarrow 0$,
	assuming uniformly bounded $\big\lVert \, \widehat{\psi}_{ij} \,\big\rVert \slash \big\lVert \,\widehat{\psi}\, \big\rVert$ terms.
\end{proposition}
\begin{proof}
    Apply the Cauchy-Schwartz inequality to \eqref{eq:definitions_2}, and use the fact that $\big\lVert \,\widehat{\mathcal{H}} \,\widehat{\psi} \,\big\rVert = \min_{\lambda \in \mathbb{R}} \big\lVert \left(\mathcal{H} - \lambda \right)
	\widehat{\psi} \,\big\rVert$.
\end{proof}
As the wavefunction approaches an eigenstate,
Proposition \ref{prop:regularity} reveals that the Hessian or Wirtinger Hessian takes a simple structure, depending only on first derivatives of the wavefunction. 
To our knowledge this fact has not been previously identified.
An important implication, to be spelled out below in Subsection \ref{sub:consequences}, is that first derivatives suffice to achieve superlinear convergence in VMC optimization, under the assumption that the true ground state lies within our parametric class.

\subsection{Gradient descent methods}
\label{sub:gradient}

The main idea in gradient descent methods is to first approximate the energy using
\begin{equation}
\label{eq:to_minimize}
\mathcal{E}_{\textup{linear}}\big[\hat{\psi}_{\bm{\theta} + \bm{\delta}}\big]
- \mathcal{E}\big[\hat{\psi}_{\bm{\theta}}\big]
= \bm{\delta}^{\ast} \bm{g} + \bm{g}^{\ast} \bm{\delta}
\end{equation}
and then choose $\bm{\delta}$ to minimize \eqref{eq:to_minimize},
plus a penalization term that keeps the update small.
The penalization term may take the form 
\begin{equation}
    \frac{\left|\bm{\delta}\right|^2}{\epsilon} \quad \text{or} \quad
    \frac{\angle \big(\,\hat{\psi}_{\bm{\theta}}, \hat{\psi}_{\bm{\theta} + \bm{\delta}}\,\big)^2}{\epsilon},
\end{equation}
where $\epsilon > 0$ is a tunable parameter.
In the first case, we are restricting the Euclidean norm $\left|\bm{\delta}\right|$ and the resulting method is standard gradient descent.
In the second case, we are restricting the angle between wavefunctions
\begin{equation}
    \angle\big(\,\hat{\psi}_{\bm{\theta}}, \hat{\psi}_{\bm{\theta} + \bm{\delta}}\,\big) =
    \arccos \frac{\big|\big<\,\hat{\psi}_{\bm{\theta}}, \hat{\psi}_{\bm{\theta} + \bm{\delta}}\,\big>\big|}
    {\big\lVert\, \hat{\psi}_{\bm{\theta}} \,\big\rVert
    \big\lVert\, \hat{\psi}_{\bm{\theta} + \bm{\delta}}\,\big\rVert}.
\end{equation}
This leads to a method called `stochastic reconfiguration' or `(quantum) natural gradient descent' 
\cite{becca2017quantum}, which has been used extensively to optimize
traditional \cite{sorella2001generalized,sorella2007weak} and more recent \cite{carleo2017solving,pfau2020ab}
VMC wavefunction ansatzes.

In a high-dimensional or infinite-dimensional Hilbert space, the angle $\angle\big(\,\hat{\psi}_{\bm{\theta}}, \hat{\psi}_{\bm{\theta} + \bm{\delta}}\,\big)$ cannot be computed exactly, so 
natural gradient descent takes advantage of the Taylor series expansion
\begin{equation}
\label{eq:imperfect}
\angle \big(\,\hat{\psi}_{\bm{\theta}}, \hat{\psi}_{\bm{\theta} + \bm{\delta}}\,\big)^2
= \bm{\delta}^{\ast} \bm{S} \bm{\delta} + 
\mathcal{O}\big(\left|\bm{\delta}\right|^3\big),
\end{equation}
where 
\begin{equation}
\bm{S}_{ij} = 
\frac{\big<\,\widehat{\psi}_i, \widehat{\psi}_j\,\big>} {\big<\,\widehat{\psi}, \widehat{\psi}\,\big>} 
\end{equation}
is a positive semidefinite matrix known as the Fubini-Study metric or quantum information metric \cite{stokes2020quantum}.
However, instead of directly using a penalization term 
\begin{equation}
    \frac{\bm{\delta}^{\ast} \bm{S} \bm{\delta}}{\epsilon},
\end{equation}
natural gradient descent uses a slightly modified penalization term 
\begin{equation}
\frac{\bm{\delta}^{\ast} \left(\bm{S} + \eta \bm{I}\right) \bm{\delta}}{\epsilon}.
\end{equation}
Here, $\eta > 0$ is a parameter that makes the matrix $\bm{S} + \eta \bm{I}$
positive definite and prevents large updates
when the Taylor series expansion \eqref{eq:imperfect} is not very accurate.

To make the preceding discussion precise, we formalize gradient descent (GD) methods as follows.
\begin{algorithm}[GD methods]{\label{alg:grad_simple}}
	Choose $\bm{\delta}$ to solve
	\begin{equation}
		\min_{\bm{\delta}}
		\left[\bm{\delta}^{\ast} \bm{g} + \bm{g}^{\ast} \bm{\delta} 
		+ \frac{\bm{\delta}^{\ast} \bm{R} \bm{\delta}}{\epsilon}\right],
	\end{equation}
	where $\bm{R} = \bm{I}$ in GD
	and $\bm{R} = \bm{S} + \eta \bm{I}$ in natural GD.
	Equivalently, set 
	\begin{equation}
	    \bm{\delta} = -\epsilon \bm{R}^{-1} \bm{g}.
	\end{equation}
\end{algorithm}

In addition to GD and natural GD,
alternative gradient descent methods such as Adam \cite{DBLP:journals/corr/KingmaB14} and AMSGrad \cite{j.2018on} have recently gained traction in the VMC community \cite{sabzevari2020accelerated,pfau2020ab,hibat2020recurrent,westerhout2020generalization}. 
These `momentum-based' methods form updates by combining the energy gradient at the current iteration and past iterations.
While such strategies are potentially promising, recent tests \cite{pfau2020ab,wierichs2020avoiding} suggest that natural GD still outperforms momentum-based methods on several challenging VMC test problems.
Therefore, we focus on GD and natural GD, leaving analysis of other gradient descent methods for future work.

\subsection{Rayleigh-Gauss-Newton method}
\label{sub:quasi}

Whereas gradient descent methods are based on a linear approximation of the energy,
we now introduce a method---which we call the Rayleigh-Gauss-Newton (RGN) method---based on the following quadratic energy approximation:
\begin{equation}
	\label{eq:quasi}
	\mathcal{E}_{\textup{quad}}\big[\widehat{\psi}_{\bm{\theta} + \bm{\delta}}\big]
	- \mathcal{E}\big[\widehat{\psi}_{\bm{\theta}}\big]
    = \bm{\delta}^{\ast} \bm{g} + \bm{g}^{\ast} \bm{\delta}
	+ \bm{\delta}^{\ast} \bm{H} \bm{\delta}.
\end{equation}
Here, $\bm{\delta}^{\ast} \bm{g}$ and $\bm{g}^{\ast} \bm{\delta}$
are the exact gradient terms,
while $\bm{\delta}^{\ast} \bm{H} \bm{\delta}$ is just one of the Hessian terms.
There is a strong practical motivation for ignoring the other Hessian term $\Re \left(\bm{\delta}^T \bm{J} \bm{\delta}\right)$,
since evaluating this term would requiring taking second derivatives of the wavefunction with respect to all pairs of parameters,
which becomes burdensome as the number of parameters grows large.

In the RGN method, we minimize the quadratic objective function \eqref{eq:quasi} plus a `natural' penalization term, as described below.
\begin{algorithm}[RGN method]{\label{alg:quad_simple}} 
	Choose $\bm{\delta}$ to solve
	\begin{equation}
	\label{eq:trust}
		\min_{\bm{\delta}}
		\left[\bm{\delta}^{\ast} \bm{g} + \bm{g}^{\ast} \bm{\delta} +
		\bm{\delta}^{\ast} \bm{H} \bm{\delta}
		+ \frac{\bm{\delta}^{\ast} \bm{R} \bm{\delta}}{\epsilon}\right],
	\end{equation}
    where $\bm{R} = \bm{S} + \eta \bm{I}$.
 	Equivalently, set 
 	\begin{equation}
 	\label{eq:consequence}
 	\bm{\delta} = -\left(\bm{H} + \epsilon^{-1} \bm{R}\right)^{-1} \bm{g}.
 	\end{equation}
\end{algorithm}
The parameter $\eta > 0$ is again chosen to make $\bm{H} + \epsilon^{-1} (\bm{S} + \eta \bm{I})$ positive definite and help prevent large parameter updates.

To our knowledge, the RGN method has not appeared before in the literature.
However, it is closely connected to 
the classical Gauss-Newton method for nonlinear least squares problems \cite{nocedal2006numerical},
which can be viewed as deriving from a similar Hessian splitting.
Also, RGN is related to
previous VMC optimization methods,
including the linear method for energy minimization \cite{becca2017quantum}
and a Gauss-Newton-like method
for variance minimization \cite{cuzzocrea2020variational}.
All these approaches can be can be described by first linearizing a class of functions
\begin{equation}
    \hat{\psi}_{\bm{\theta} + \bm{\delta}}
    \approx 
    \hat{\psi} + \sum\nolimits_i \bm{\delta}_i \hat{\psi}_i,
\end{equation}
and then minimizing a nonlinear loss function
applied to the linearized function class
\begin{equation}
    \min_{\bm{\delta}} \,\mathcal{L}\big[\hat{\psi}
    + \sum\nolimits_i \bm{\delta}_i \hat{\psi}_i
    \big].
\end{equation}

For example, in the linear method for VMC,
one first linearizes the intermediate-normalized wavefunction $\widehat{\psi}_{\bm{\theta} + \bm{\delta}}$
and then minimizes
\begin{equation}
\label{eq:rayleigh}
\mathcal{E}\left[\widehat{\psi} 
+ \sum\nolimits_i \bm{\delta}_i \widehat{\psi}_i\right]
- \mathcal{E}\left[\widehat{\psi}\right]
= \frac{\bm{\delta}^{\ast} \bm{g} + \bm{g}^{\ast} \bm{\delta} + \bm{\delta}^{\ast} \bm{H} \bm{\delta}}
{1 + \bm{\delta}^{\ast} \bm{S} \bm{\delta}},
\end{equation}
plus an additional penalization term.
Minimization of \eqref{eq:rayleigh} is equivalent to solving the generalized eigenvalue problem
\begin{equation}
    \begin{pmatrix} 0 & \bm{g}^{\ast} \\ \bm{g} & \bm{H} \end{pmatrix}
    \begin{pmatrix} 1 \\ \bm{\delta} \end{pmatrix}
    = \lambda \begin{pmatrix} 1 & \bm{0} \\ \bm{0} & \bm{S} \end{pmatrix}
    \begin{pmatrix} 1 \\ \bm{\delta} \end{pmatrix}
\end{equation}
for the smallest eigenvalue-eigenvector pair
\cite{gubernatis2016quantum,becca2017quantum}.
As a penalization term, the matrix $\bm{H}$ is padded
with a diagonal matrix $\epsilon^{-1} \bm{I}$,
which is similar to the penalization term used in GD.

The linear method has been observed to yield fast asymptotic convergence in VMC applications with small parameter sets.
However, extending the linear method to larger parameter sets is an ongoing challenge \cite{neuscamman2012optimizing,sabzevari2020accelerated}.
Motivated by the linear method's potential for fast asymptotic convergence, we introduce RGN as an updated strategy with improved convergence.
RGN differs from the linear method in two ways, as detailed below.

First, instead of approximating the energy using the formula \eqref{eq:rayleigh},
RGN uses the quadratic approximation
\begin{equation}
	\mathcal{E}_{\textup{quad}}\big[\widehat{\psi}_{\bm{\theta} + \bm{\delta}}\big]
	- \mathcal{E}\big[\widehat{\psi}_{\bm{\theta}}\big]
	= \bm{\delta}^{\ast} \bm{g} + \bm{g}^{\ast} \bm{\delta}
	+ \bm{\delta}^{\ast} \bm{H} \bm{\delta}.
\end{equation}
This quadratic approximation agrees with \eqref{eq:rayleigh} up to $\mathcal{O}\big(\left|\bm{\delta}\right|^3\big)$ terms,
but 
it only requires the solution of a linear system instead of a generalized eigenvalue problem. 
Although the parametrizations considered in our numerical experiments below are small enough so that neither of these linear algebra routines imposes a computational bottleneck, the distinction may become important for large parameter sets.
For example, 
\cite{neuscamman2012optimizing} introduced a matrix-free approach for solving the linear system in stochastic reconfiguration, but a similar scheme for the generalized eigenvalue problem has not achieved the same success \cite{zhao2017blocked}.

Second, RGN uses a `natural' penalization term $\epsilon^{-1} \bm{\delta}^{\ast}\left(\bm{S} + \eta \bm{I}\right) \bm{\delta}$, which differs from the penalization term used in the linear method.
Because of the penalization,
the linear method converges as $\epsilon \rightarrow 0$ to give standard GD updates.
In contrast,
RGN converges as $\epsilon \rightarrow 0$ to give natural GD updates,
which are more efficient than standard GD updates
when optimizing many VMC wavefunction ansatzes \cite{pfau2020ab,wierichs2020avoiding}.

\subsection{Convergence rate analysis}{\label{sub:consequences}}

GD, natural GD, and RGN can all be presented in the standardized form
\begin{equation}
\label{eq:general}
\bm{P}^i \left(\bm{\theta}^{i+1} - \bm{\theta}^i\right)
= - \bm{g}\left(\bm{\theta}^i\right),
\qquad i = 1, 2, \ldots
\end{equation}
Here, the parameter update $\bm{\theta}^{i+1} - \bm{\theta}^i$ is written as the solution to a linear system involving a positive definite preconditioning matrix $\bm{P}^i$ and a negative energy gradient $-\bm{g}\left(\bm{\theta}^i\right)$.
Table \ref{table:preconditioners} shows the different preconditioners corresponding to the different optimization approaches.

\begin{table}[h!]
\centering
\begin{tabular}{l|l}
    Method & Preconditioner $\bm{P}$ \\
    \hline
    Gradient descent & $\epsilon^{-1} \bm{I}$ \\
    Natural gradient descent ~ & $\epsilon^{-1} \left(\bm{S} + \eta \bm{I}\right)$ \\
    Rayleigh-Gauss-Newton &
    $\bm{H} + \epsilon^{-1} \left(\bm{S} + \eta \bm{I}\right)$
\end{tabular}
\caption{Different preconditioners for energy minimization.}
\label{table:preconditioners}
\end{table}

To help quantify the efficiency of the various optimization methods,
Proposition \ref{prop:rate} considers a general sequence of positive definite preconditioners $\bm{P}^1, \bm{P}^2, \ldots$ 
and derives sharp asymptotic bounds on the resulting energies $\mathcal{E}\left[\psi_{\bm{\theta}^1}\right], \mathcal{E}\left[\psi_{\bm{\theta}^2}\right], \ldots$.
Proposition \ref{prop:rate} is based on standard optimization theory (e.g., \cite{nocedal2006numerical}),
but here we extend this theory to the complex-valued wavefunctions
often used in VMC.

\begin{proposition}{\label{prop:rate}} 
    Consider the parameter updates
    $\bm{P}^i \left(\bm{\theta}^{i+1} - \bm{\theta}^i\right) = - \bm{g}\left(\bm{\theta} ^i\right)$.
    Assume $\bm{\theta}^1, \bm{\theta}^2, \ldots$ converges to a local energy minimizer $\bm{\theta}^{\ast}$,
    and the Hessian or Wirtinger Hessian is positive definite at $\bm{\theta}^{\ast}$.
    Then,
    \begin{multline}
        \limsup_{i \rightarrow \infty} \frac{\mathcal{E}\left[\psi_{\bm{\theta}^{i+1}}\right] 
        - \mathcal{E}\left[\psi_{\bm{\theta}^{\ast}}\right]}
        {\mathcal{E}\left[\psi_{\bm{\theta}^{i}}\right] 
        - \mathcal{E}\left[\psi_{\bm{\theta}^{\ast}}\right]} \\
        \leq 
        \limsup_{i \rightarrow \infty} 
        \left\lVert \bm{I} -
        \left(\bm{H} + \bm{J}\right)^{\frac{1}{2}}
        \bm{P}_i^{-1} \left(\bm{H} + \bm{J}\right)^{\frac{1}{2}} \right\rVert_2^2
    \end{multline}
    or
    \begin{multline}
        \limsup_{i \rightarrow \infty} \frac{\mathcal{E}\left[\psi_{\bm{\theta}^{i+1}}\right] 
        - \mathcal{E}\left[\psi_{\bm{\theta}^{\ast}}\right]}
        {\mathcal{E}\left[\psi_{\bm{\theta}^{i}}\right] 
        - \mathcal{E}\left[\psi_{\bm{\theta}^{\ast}}\right]} \\
        \leq 
        \limsup_{i \rightarrow \infty} 
        \Big\lVert \bm{I} -
        \bigl(\begin{smallmatrix}
        \bm{H} & \bm{J} \\
        \overline{\bm{J}} & \overline{\bm{H}}
        \end{smallmatrix}\bigr)^{\frac{1}{2}}
        \bigl(\begin{smallmatrix}
        \bm{P}_i & \bm{0} \\
        \bm{0} & \overline{\bm{P}_i}
        \end{smallmatrix}\bigr)^{-1}
        \bigl(\begin{smallmatrix}
        \bm{H} & \bm{J} \\
        \overline{\bm{J}} & \overline{\bm{H}}
        \end{smallmatrix}\bigr)^{\frac{1}{2}}
        \Big\rVert_2^2
    \end{multline}
in the real and complex cases, respectively, where $\bm{H} = \bm{H}(\bm{\theta}^{\ast})$ and $\bm{J} = \bm{J}(\bm{\theta}^{\ast})$. 
\end{proposition}
\begin{proof}
    See Appendix \ref{sec:proofs}.
\end{proof}

The convergence rate in Proposition \ref{prop:rate} depends on a matrix $\bm{J}$ which vanishes
at the ground state.
Therefore,
If the RGN method is applied with penalization parameters $\epsilon = \epsilon^i$ tending to infinity and wavefunctions $\psi_{\bm{\theta}^i}$ approaching the ground state, the rate of convergence is \emph{superlinear}, i.e.,
\begin{equation}
\limsup_{i \rightarrow \infty} \frac{\mathcal{E}\left[\psi_{\bm{\theta}^{i+1}}\right] 
- \mathcal{E}\left[\psi_{\bm{\theta}^{\ast}}\right]}
{\mathcal{E}\left[\psi_{\bm{\theta}^{i}}\right] 
- \mathcal{E}\left[\psi_{\bm{\theta}^{\ast}}\right]} = 0. 
\end{equation}
In practice, our parametric class does not usually contain the \emph{exact} ground state for $\mathcal{H}$, but if $\epsilon$ is large and the energy minimizer is close to the ground state, then Proposition \ref{prop:rate} still quantifies a fast linear convergence rate for RGN.
In the numerical experiments presented in Section \ref{sec:results},
we achieve such a fast convergence rate by gradually moving the parameter $\epsilon$ closer to zero as we make progress in optimizing the wavefunction.

\section{VMC sampling analysis}{\label{sec:sampling}}

In this section, we review VMC sampling and
prove a vanishing-variance principle that quantifies
the sampling error in the estimated energies and gradients.
Then, we discuss challenges in VMC sampling
and motivate strategies to improve the sampling.

\subsection{VMC sampling}

VMC requires quantities such as $\mathcal{E}$, $\bm{g}$, $\bm{S}$, and $\bm{H}$ that are constructed as sums or integrals over a high-dimensional 
or infinite-dimensional state space.
To compute such quantities, 
VMC relies on the power of Monte Carlo sampling.
In VMC, we first generate 
a large number of samples $\bm{\sigma}_1, \bm{\sigma}_2, \ldots, \bm{\sigma}_T$ 
from the normalized wavefunction density
\begin{equation}
\rho(\bm{\sigma}) = \frac{\left|\psi(\bm{\sigma})\right|^2}{\left<\psi, \psi\right>}
\end{equation}
using an appropriate Markov chain Monte Carlo (MCMC, \cite{liu2008monte}) sampler.
Then we approximate $\mathcal{E}$, $\bm{g}$, $\bm{S}$, and $\bm{H}$ using the following estimators:
\begin{align}
\label{eq:formula1}
    &\hat{\mathcal{E}}
    = \E_{\hat{\rho}}\!\left[E_{L}(\bm{\sigma})\right], \\
    &\hat{\bm{g}}_i 
    = \Cov_{\hat{\rho}}\!\left[\bm{\nu}_i(\bm{\sigma}), E_{L}(\bm{\sigma})\right], \\
    &\hat{\bm{S}}_{ij} 
    = \Cov_{\hat{\rho}}\!\left[\bm{\nu}_i(\bm{\sigma}), \bm{\nu}_j(\bm{\sigma})\right], \\
\label{eq:formula4}
    &\hat{\bm{H}}_{ij} = \Cov_{\hat{\rho}}\!\left[\bm{\nu}_i(\bm{\sigma}), E_{L,j}(\bm{\sigma})\right]
    -\hat{\bm{g}}_i \E_{\hat{\rho}}\!\left[\bm{\nu}_j(\bm{\sigma})\right]
    - \hat{\mathcal{E}} \hat{\bm{S}}_{ij}.
\end{align}
Here, $\E_{\hat{\rho}}$ and $\Cov_{\hat{\rho}}$ 
denote expectations and covariances with respect to the empirical measure
\begin{equation}
\hat{\rho} = \frac{1}{T} \sum_{t=1}^T \delta_{\bm{\sigma}_t},
\end{equation}
and we have introduced the functions
\begin{align}
    E_{L}(\bm{\sigma}) 
    = \frac{\mathcal{H} \psi (\bm{\sigma})}{\psi(\bm{\sigma})},
    && E_{L,i}(\bm{\sigma}) 
    = \frac{\mathcal{H} \partial_{\bm{\theta}_i} \psi (\bm{\sigma})}{\psi(\bm{\sigma})}, \\
    \bm{\nu}_i(\bm{\sigma}) = \frac{\partial_{\bm{\theta}_i} \psi(\bm{\sigma}) }{\psi(\bm{\sigma})}. &&
\end{align}
The functions $E_{L}$ and $\bm{\nu}_i$ are known as the local energy and logarithmic derivative, respectively.

Next we state the vanishing-variance principle, which quantifies the asymptotic variance of several VMC estimators of interest.

\begin{proposition}{\label{prop:zero_variance}}
    Assume the MCMC sampler is geometrically ergodic with respect to $\rho$,
    and for some $\epsilon > 0$, 
    $\E_{\rho}\!\left|E_{L}(\bm{\sigma})\right|^{4 + \epsilon} < \infty$
    and 
    $\sup_i \E_{\rho}\!\left|\nu_i(\bm{\sigma})\right|^{4 + \epsilon} < \infty$.
    Then, as $T \rightarrow \infty$,
    \begin{align}
        & \sqrt{T} \big(\,\hat{\mathcal{E}}_T - \mathcal{E}\big)
        \stackrel{\mathcal{D}}{\rightarrow}
        \mathcal{N}\left(0, v^2\right), \\
        & \sqrt{T}\big(\,\hat{\bm{g}}_T - \bm{g}\big) \stackrel{\mathcal{D}}{\rightarrow} \mathcal{N}\left(\bm{0}, \bm{\Sigma}\right),
    \end{align}
    where the asymptotic variances $v^2$ and $\bm{\Sigma}$ are given by
    \begin{align}
    \label{eq:var1}
        & \begin{aligned}
        v^2 & =\sum_{t=0}^{\infty}
        \Cov_{\bm{\sigma}_0 \sim \rho} \left[E_{L}(\bm{\sigma}_0), E_{L}(\bm{\sigma}_t)\right] \\
        & + \sum_{t=1}^{\infty} \Cov_{\bm{\sigma}_0 \sim \rho} \left[E_{L}(\bm{\sigma}_t), E_{L}(\bm{\sigma}_0)\right],
        \end{aligned}
        \\
    \label{eq:var2}
        & \begin{aligned}
        \bm{\Sigma}_{ij} &=
        \sum_{t=0}^{\infty} \Cov_{\bm{\sigma}_0 \sim \rho}
        \left[\bm{g}^{\prime}_i(\bm{\sigma_0}), 
        \bm{g}^{\prime}_j(\bm{\sigma}_t)\right] \\
        &+
        \sum_{t=1}^{\infty} \Cov_{\bm{\sigma}_0 \sim \rho}
        \left[\bm{g}^{\prime}_i(\bm{\sigma_t}), 
        \bm{g}^{\prime}_j(\bm{\sigma}_0)\right],
        \end{aligned},
    \end{align}
    and
    $\bm{g}^{\prime}$ is defined as 
    \begin{equation} \bm{g}^{\prime}(\bm{\sigma})
        = \overline{\left( \bm{\nu}(\bm{\sigma}) - \E_{\bm{\sigma}^{\prime} \sim \rho}\! \left[\bm{\nu}
        \left(\bm{\sigma}^{\prime}\right) \right] \right)}
        \left(E_{L}(\bm{\sigma}) - \mathcal{E}\right).
    \end{equation}
\end{proposition}
\begin{proof}
    See Appendix \ref{sec:proofs}.
\end{proof}

As a major takeaway from Proposition \ref{prop:zero_variance}, 
the energy and energy gradient both have zero variance,
i.e., $v^2 = 0$ and $\bm{\Sigma} = \bm{0}$,
when the local energy $E_L$ is constant,
as occurs at any eigenstate of $\mathcal{H}$.
Proposition \ref{prop:zero_variance} can thus be viewed as a robust and quantitative extension of the classic zero-variance principle \cite{gubernatis2016quantum,becca2017quantum} of VMC.
The vanishing-variance principle is robust, since it holds when the wavefunction is not an eigenstate, and quantitative, since it gives a precise formula for the asymptotic variance of the energy 
and energy gradient estimators.


\subsection{Improving estimation quality} \label{sub:improving}

Near an eigenstate, the vanishing-variance principle ensures the relative accuracy of VMC estimated energies.
However, away from an eigenstate, VMC estimated energies and energy gradients can have a high variance and change erratically over the course of VMC estimation
\cite{westerhout2020generalization,park2020geometry}.
Therefore, variance reduction strategies are needed to ensure VMC's success.

Proposition \ref{prop:zero_variance}
suggests three strategies for reducing the variance.
The first strategy
is to increase the number of Monte Carlo samples.
We can do this either by running one MCMC sampler for a long time or by running many MCMC samplers in parallel and combining samples. 
The parallel sampling approach often leads to computational advantages,
since vectorized code runs quickly on modern computers and
MCMC samplers can be run on multiple nodes/cores
to further cut down on the runtime.
In the numerical experiments 
in Section \ref{sec:results},
we run $50$ MCMC samplers per CPU core
and use $48$ CPU cores, 
thus generating $2400$ parallel MCMC samplers.

The second variance reduction strategy is to reduce
correlations among the samples
$\bm{\sigma}_1, \bm{\sigma}_2, \ldots, \bm{\sigma}_T$
by applying a fast-mixing MCMC method such as parallel tempering
\cite{swendsen1986replica}.
Parallel tempering introduces interacting MCMC samplers
that target different densities
\begin{equation}
    \rho_i(\bm{\sigma}) \propto \rho(\bm{\sigma})^{i \slash m}, \qquad i = 0, 1, \ldots, m,
\end{equation}
Periodically, the samplers targeting adjacent densities $\rho_i$ and $\rho_{i+1}$
swap positions according to a
Metropolis acceptance probability \cite{metropolis1953equation},
which improves the mixing time for each of the samplers.
Lastly, the samplers targeting $\rho_m$ are used for estimating $\mathcal{E}$, $\bm{g}$, $\bm{S}$, and $\bm{H}$.
Parallel tempering 
has reduced correlations in challenging
VMC test problems in the past \cite{choo2018symmetries,park2020geometry},
and in Subsection \ref{sub:sampling} we apply parallel tempering to improve the sampling for XXZ models on large lattices.

The third variance reduction strategy
is to directly alter the VMC update formula to improve its stability.
For example, \cite{pfau2020ab} and \cite{hermann2020deep} 
use an alternative gradient estimator in which the most extreme local energy
values are adjusted to be 
closer to the median.
Similarly, \cite{luo2019backflow}
rounds all positive gradient entries to $+1$, round all negative gradient entries to $-1$, and then assign a random independent magnitude to each entry.
The approaches \cite{pfau2020ab,hermann2020deep,luo2019backflow} all improve the stability of VMC updates, 
but they discard gradient information
that could potentially be helpful.
Therefore, we adopt a slightly different stabilization approach in Section \ref{sec:results}.
At each iteration, we check that the parameter update is less than twice as large as the previous parameter update (in Euclidean norm).
If not, we shrink $\epsilon$ in half repeatedly until the parameter update is sufficiently small.
This stabilization
code eliminates the most erratic
parameter updates in our experiments.
The code is rarely triggered for TFI models
(just $0$--$2$ times per $1000$ updates),
but it is more commonly triggered for XXZ models
($9$--$34$ times per $1000$ updates).

\section{Numerical experiments}{\label{sec:results}}

To test the performance of VMC optimization and sampling methods, we estimate the ground-state energies for
the transverse-field Ising (TFI) and XXZ models 
on 1-D and 2-D lattices with periodic boundary conditions.
These models are specified by the Hamiltonians 
\begin{align}
\label{eq:tfi}
    & \mathcal{H}_{\textup{TFI}} =
    -\sum_{i \sim j} \sigma_i^z \sigma_j^z - h \sum_i \sigma_i^x,  \\
\label{eq:xxz}
    & \mathcal{H}_{\textup{XXZ}} =
    - \Delta \sum_{i \sim j} \sigma_i^z \sigma_j^z
    + \sum_{i \sim j} [\sigma_i^y \sigma_j^y - \sigma_i^x \sigma_j^x],
\end{align}
where $h > 0$ and $\Delta > 0$ are positive-valued parameters.
The XXZ model is sometimes alternatively defined as
\begin{equation}
    \mathcal{H}_{\textup{XXZ}} =
    \Delta \sum_{i \sim j} \sigma_i^z \sigma_j^z
    + \sum_{i \sim j} [\sigma_i^x \sigma_j^x
    + \sigma_i^y \sigma_j^y],
\end{equation}
which is a unitary transformation of \eqref{eq:xxz},
assuming a bipartite lattice.
As a consequence of the Perron-Frobenius theorem
and translational symmetry,
the models 
\eqref{eq:tfi} and \eqref{eq:xxz}
both admit
unique,  nonnegative, translationally invariant ground-state wavefunctions.
For 1-D lattices but not 2-D lattices, the
ground-state wavefunctions are known exactly \cite{bethe1931theorie,pfeuty1970one}.

As a wavefunction ansatz, we use a restricted Boltzmann machine (RBM), which can be written as
\begin{equation}
    \psi_{\bm{w},\bm{b}} (\bm{\sigma}) = \prod_{i=1}^{\alpha} \prod_\mathcal{T} \cosh\Bigg(\sum_j \bm{w}_{ij} \left(\mathcal{T} \bm{\sigma}\right)_j + \bm{b}_i \Bigg).
\end{equation}
Here, $\alpha$ is the hidden-variable density that controls the number of parameters, $\mathcal{T}$ ranges over the translation operators on the periodic lattice, and $\bm{w}$ and $\bm{b}$ are vectors of complex-valued parameters, called \emph{weights} and \emph{biases}. 
This ansatz is an example of a two-layer neural network and is a simplification of the RBM ansatz used for VMC optimization in \cite{carleo2017solving}.
The ansatz involves $\alpha \left(n + 1\right)$ parameters,
where $n$ is the number of spins and we set 
$\alpha = 5$ for all of our numerical experiments.
We report additional implementation details in Appendix \ref{sec:computations}.

\subsection{Comparing optimization methods}

To compare different VMC optimizers in the noiseless setting,
we apply VMC to TFI model on a $10 \times 1$ lattice, which is small enough so that $\mathcal{E}$, $\bm{g}$, $\bm{S}$, and $\bm{H}$ can be computed by exact summation without the need for Monte Carlo sampling.
Figure \ref{fig:comparison} evaluates the performance of different optimizers in this setting,
i.e., with deterministic parameter updates.
The figure shows that RGN
leads to faster convergence and lower errors than GD, natural GD, and the linear method.

\begin{figure}[h!]
    \centering
    \includegraphics[scale=.35]{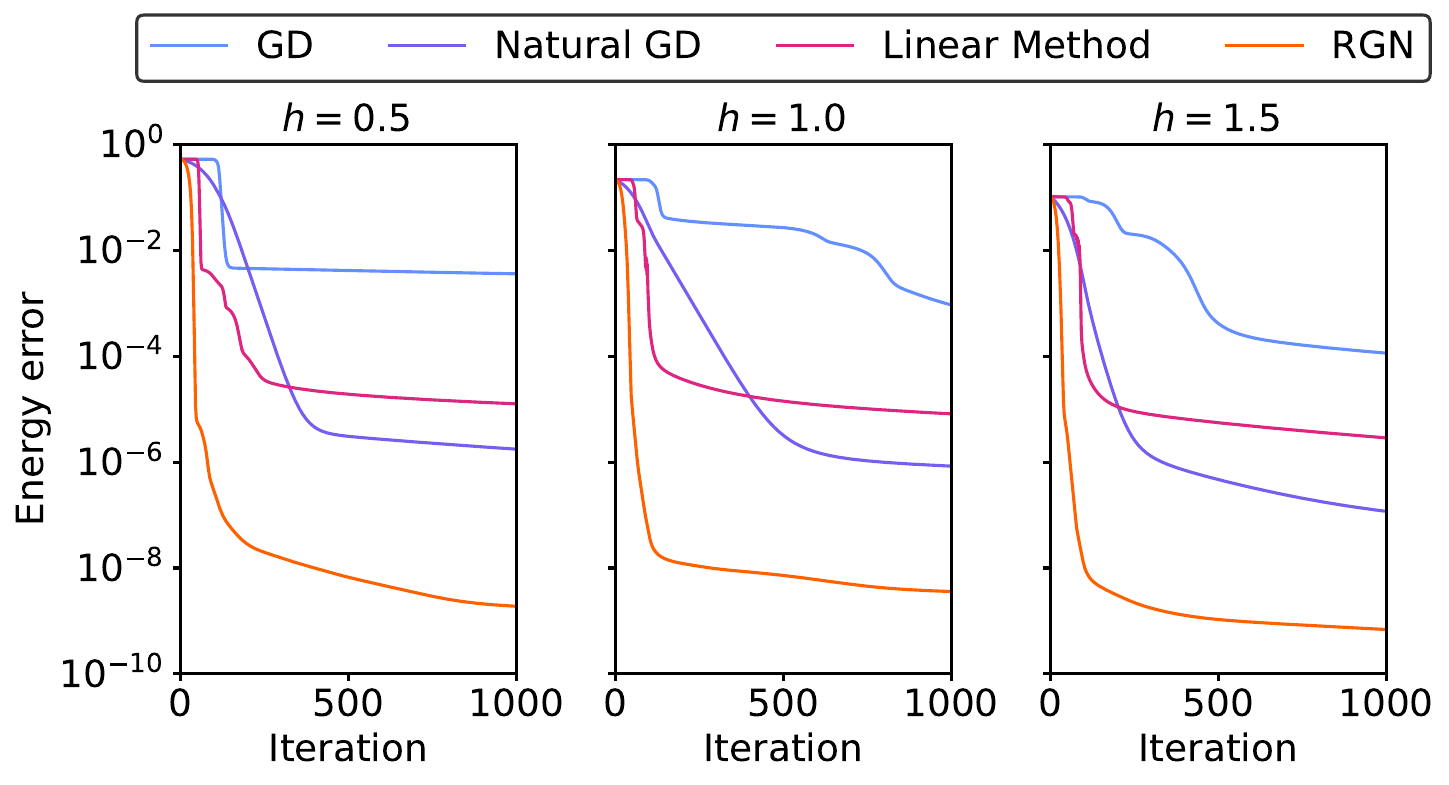}
    \caption{RGN achieves low energy errors in ferromagnetic ($h = 0.5$, left),
    transitional ($h = 1.0$, center),
    and paramagnetic ($h = 1.5$, right) regimes.
    Plot shows relative error in ground-state energy estimates.}
    \label{fig:comparison}
\end{figure}

In light of Proposition \ref{prop:rate}, 
we expect the most rapid energy convergence when the preconditioner is close to the true Hessian $\left(\begin{smallmatrix}
    \bm{H} & \bm{J} \\
    \overline{\bm{J}} & \overline{\bm{H}}
    \end{smallmatrix}\right)$. 
Indeed, Figure \ref{fig:hessian} confirms that the Hessian approximation
used in RGN closely approximates the true Hessian, in concordance with the fast observed convergence rate.

\begin{figure}[h!]
    \centering
    \includegraphics[scale=.35]{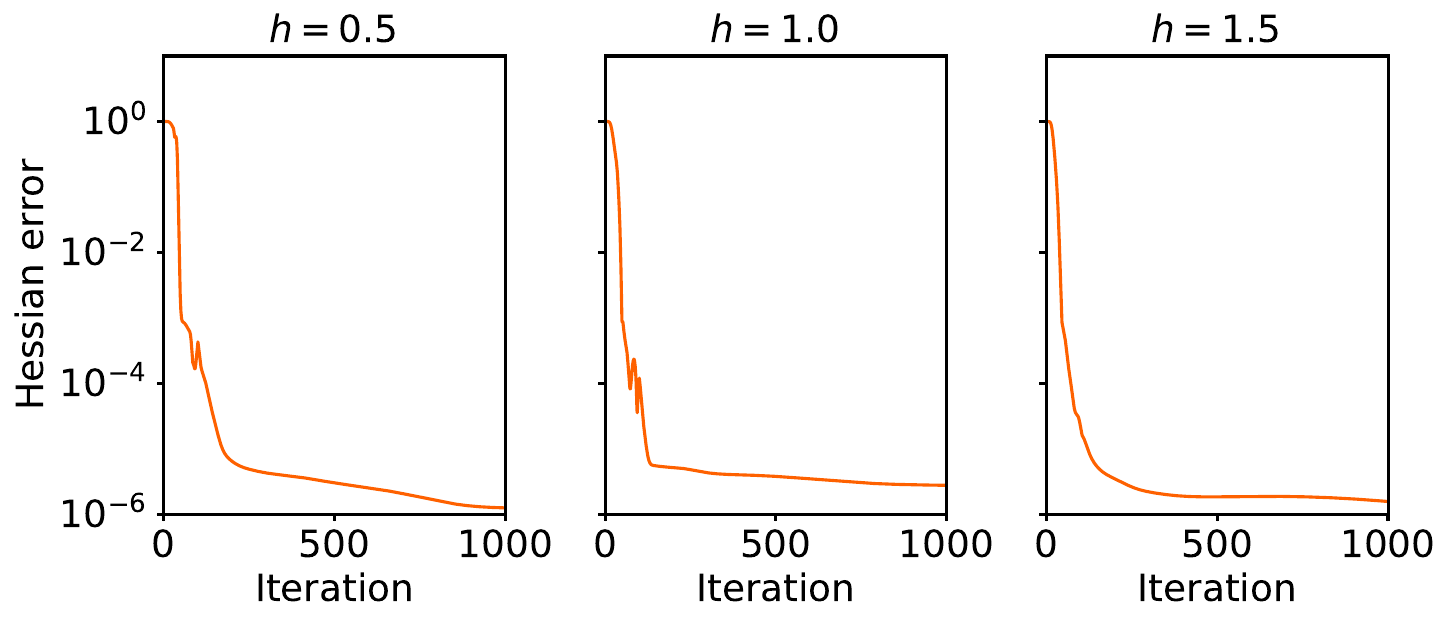}
    \caption{RGN achieves accurate Hessian approximations
    with relative errors $< 10^{-5}$ for most iterations.
    Plot shows relative error $\left|\left(\begin{smallmatrix}
    \bm{0} & \bm{J} \\
    \overline{\bm{J}} & \overline{\bm{0}}
    \end{smallmatrix}\right)\right| \slash 
    \left|\left(\begin{smallmatrix}
    \bm{H} & \bm{J} \\
    \overline{\bm{J}} & \overline{\bm{H}}
    \end{smallmatrix}\right)\right|$ computed at each iteration.}
    \label{fig:hessian}
\end{figure}

\subsection{Challenges in VMC sampling} \label{sub:sampling}

We next apply VMC to larger lattices by incorporating MCMC sampling.
For the TFI model, we initialize the MCMC samplers from a configuration chosen uniformly at random and propose random updates based on flipping a single spin.
For the XXZ model, we confine the MCMC samplers to
`balanced' configurations for which the magnetization is the same on both components of the bipartite lattice,
since the ground-state wavefunction is only supported on these configurations.
We initialize from a random balanced configuration and propose random balanced updates based on flipping two spins.

At every new optimization step,
the MCMC samplers are continued from the final configurations at the previous step.
The MCMC samplers are then run for $20 \times n$ time steps,
and the local energies and logarithmic derivatives are evaluated
at intervals of $n$ time steps, where $n$ denotes the number of spins.

The MCMC samplers
are guaranteed to mix quickly when sampling the ground-state wavefunctions for the TFI model at $h = \infty$ or the XXZ model at $\Delta = -1$. 
For these extreme parameter settings, every Metropolis proposal is accepted, the relaxation time for the TFI sampler is $n \slash 2$ \cite{levin2017markov}, and the relaxation time for the XXZ sampler is $n \slash 4$ \cite{diaconis1988group}.
Yet, there is no guarantee that the MCMC samplers remain efficient for the $h$ and $\Delta$ values more reasonably encountered. 

Indeed, 
the RGN and natural GD optimizers encounter difficulties when calculating ground-state energies for the XXZ model,
as shown in Figure \ref{fig:tempering}.
Initially during the optimization,
the RGN energies decrease quickly,
but at iteration $360$ ($\Delta = 1.5$) or $370$ ($\Delta = 1.0$),
the energies exhibit a large spike,
which persists over roughly $1000$ optimization steps.
The natural GD energies exhibit a spike later,
during iterations $1500$--$4000$ ($\Delta = 1.5$),
which makes sense because the 
natural GD optimizer converges more slowly than the RGN optimizer overall.

\begin{figure}[h!]
    \centering
    \includegraphics[scale=.35]{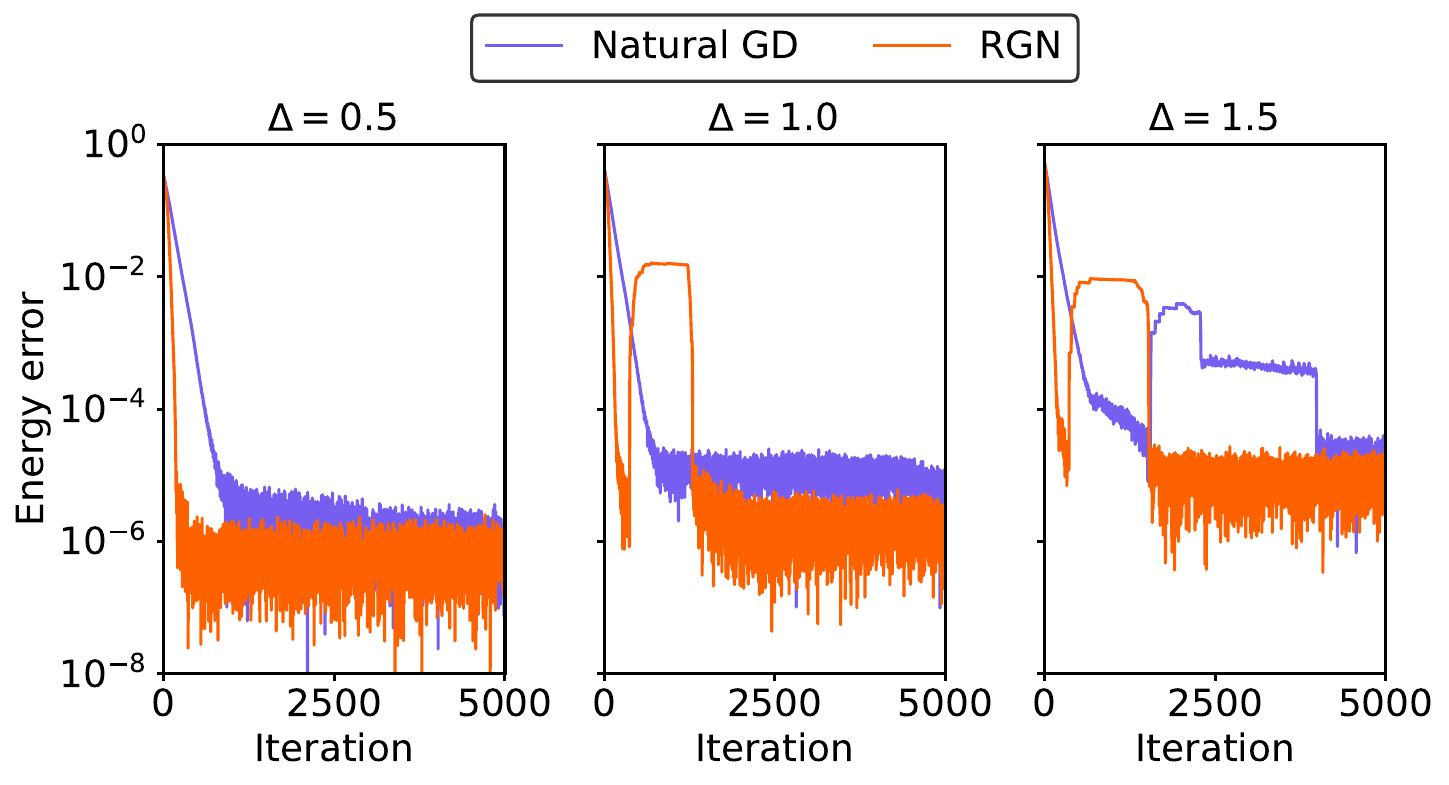}
    \caption{VMC can lead to energy spikes if direct MCMC sampling is used.
    Plot shows relative error in ground-state energy estimates
    for an XXZ model on a $100 \times 1$ lattice.}
    \label{fig:tempering}
\end{figure}

The energy spikes are a major difference between exact VMC energies and energy estimates from MCMC sampling.
The exact energies change slowly and continuously, 
as seen in Figure \ref{fig:comparison}.
However, the MCMC energy estimates can spike
if a slowly-mixing MCMC sampler moves between metastable regions of configuration space.
Indeed, Figure \ref{fig:explain_spikes} establishes that all $2400$ MCMC samples typically lie in the ferromagnetic region of configuration space.
At the onset of the energy spikes, a few MCMC samplers ($5$--$30$) enter the antiferromagnetic region of configuration space, encountering
high densities and high local energies that have not been experienced before.
The densities and local energies are extremely large due to generalization error,
and the optimizers require $1000+$ iterations to respond to the new MCMC data and eliminate the spikes.

\begin{figure}[h!]
    \centering
    \includegraphics[scale=.35]{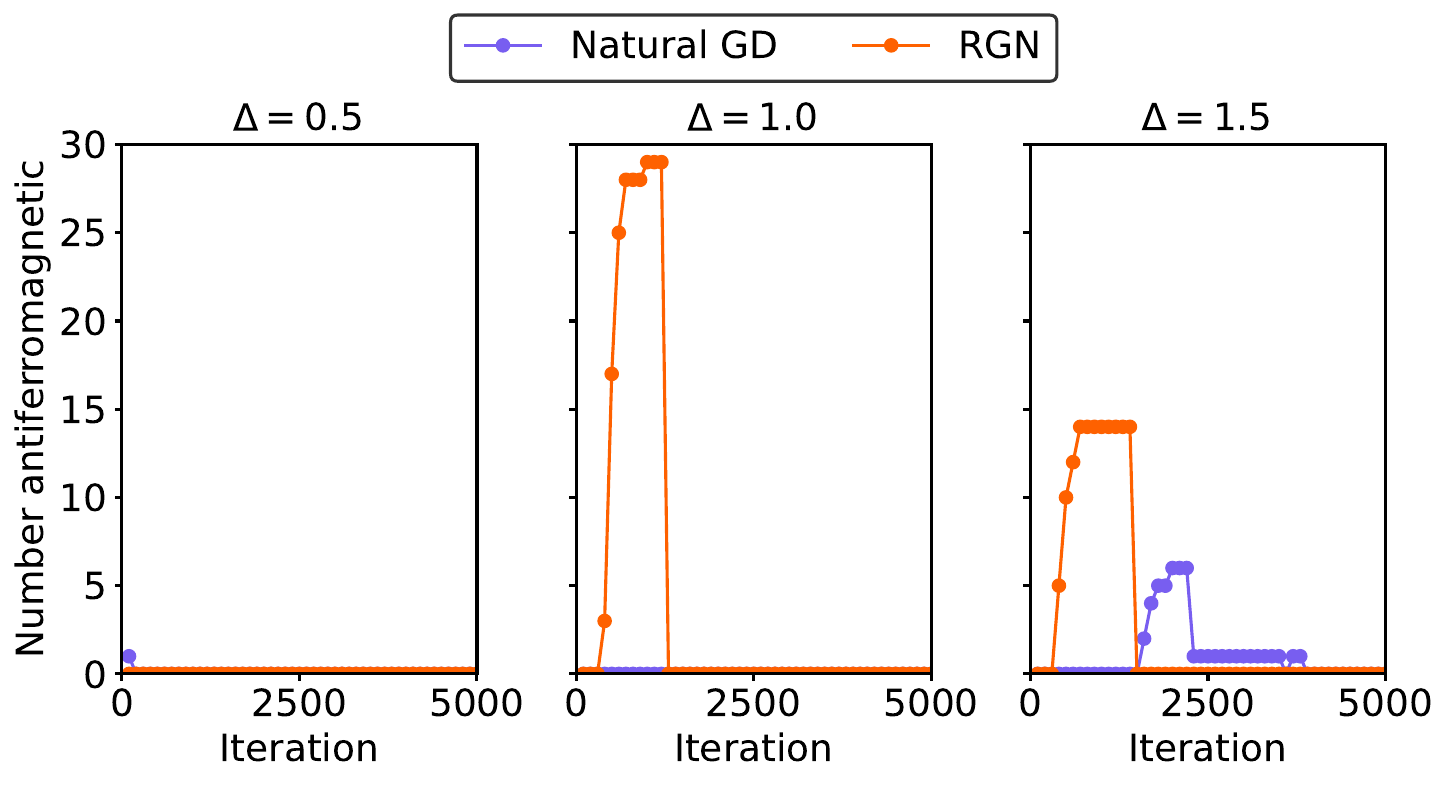}
    \caption{Energy spikes occur when a few MCMC samplers enter the antiferromagnetic region of configuration space, defined by $\sum_{i \sim j} \sigma_i \sigma_j < 0$.
    Plot shows number of samplers in the antiferromagnetic region, evaluated at every $100$ iterations.}
    \label{fig:explain_spikes}
\end{figure}

To improve the sampling for XXZ models, we apply the parallel tempering method described in Section \ref{sub:improving} using fifty target densities.
Parallel tempering speeds up the mixing of the MCMC samplers
and reduces the magnitude and longevity of the energy spikes, as shown in Figure \ref{fig:fixed}.
With parallel tempering, the same number of MCMC samples ($2400 \times 20$) are generated per iteration as in direct MCMC sampling, but the quality of the samples is much higher.
The high-quality sampling reduces the generalization error and improves the overall stability of VMC.

\begin{figure}[h!]
    \centering
    \includegraphics[scale=.35]{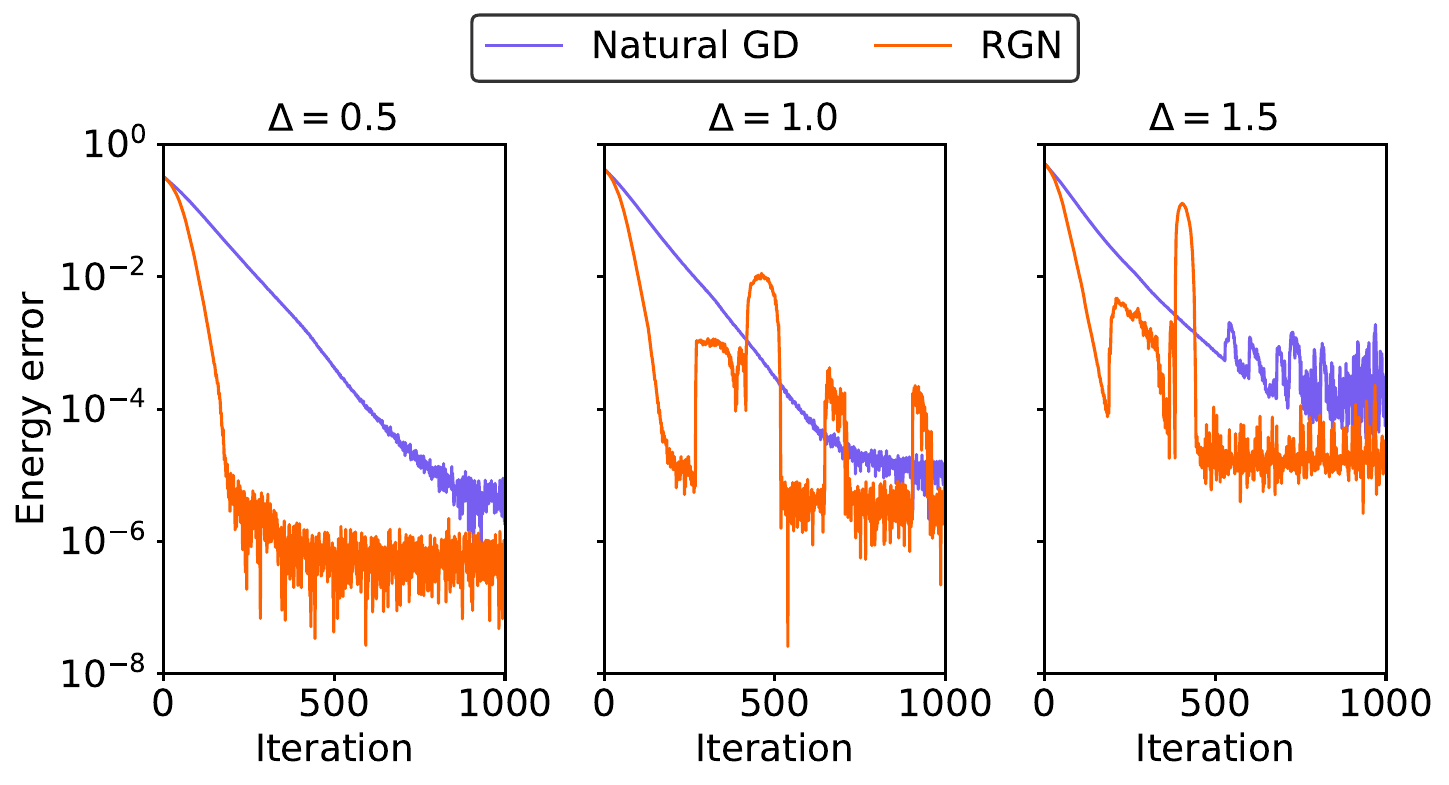}
    \caption{VMC recovers more quickly from energy spikes if parallel tempering is used.}
    \label{fig:fixed}
\end{figure}

\subsection{Results for larger systems}

Lastly, we apply VMC to estimate ground-state energies for TFI and XXZ models on lattices with up to $400$ spins.
We train highly accurate VMC wavefunctions for these large lattices
by using RGN and (for XXZ models) parallel tempering.
For 1-D systems, we compare the estimated ground-state energies against the exact energies in Table \ref{tab:results1}.
For 2-D systems, we report the estimated energies themselves in Table \ref{tab:results2}.

\begin{table}[h!]
    \centering
    \begin{tabular}{r|ccc|}
    \multicolumn{1}{c}{~} & \multicolumn{3}{c}{$200 \times 1$ TFI model} \\
    \cline{2-4}
    $~$ & $h = 0.5$ & $h = 1.0$ & $h = 1.5$ \\
    Natural GD ~ & $3.9 \times 10^{-5}$
    & $1.4 \times 10^{-4}$  & $8.5 \times 10^{-8}$  \\
    RGN ~ & $\bm{1.0 \times 10^{-9}}$
    & $\bm{2.9 \times 10^{-6}}$ & $\bm{1.6 \times 10^{-9}}$ \\
    \cline{2-4}
    \multicolumn{4}{c}{~}
    \\
    \multicolumn{1}{c}{~} & \multicolumn{3}{c}{$100 \times 1$ XXZ model} \\
    \cline{2-4}
    ~ & $\Delta = 0.5$ & 
    $\Delta = 1.0$ & $\Delta = 1.5$ \\
    Natural GD ~ & $3.9 \times 10^{-6}$
    & $1.2 \times 10^{-5}$ & $5.4 \times 10^{-5}$
     \\
    RGN ~ & $\bm{2.5 \times 10^{-7}}$
    & $\bm{3.3 \times 10^{-6}}$ & $\bm{2.0 \times 10^{-5}}$ \\
    \cline{2-4}
    \end{tabular}
    \caption{Relative errors in ground-state energy estimates after $1000$ iterations.
    Lower errors are marked in bold.}
    \label{tab:results1}
\end{table}

\begin{table}[ht!]
    \centering
    \begin{tabular}{r|ccc|}
    \multicolumn{1}{c}{~} & \multicolumn{3}{c}{$20 \times 20$ TFI model, Natural GD} \\
    \cline{2-4}
    ~ & $h = 2.0$ & $h = 3.0$ & $h = 4.0$ \\
    iteration $200$~ & $-2.\bm{3375353}$
    &  $-3.1\bm{899006}$ & $-4.133\bm{7097}$ \\
    iteration $1000$~ & $-2.\bm{5113061}$
    &  $-3.1\bm{950035}$ & $-4.133\bm{8352}$ \\
    \cline{2-4}
    \multicolumn{4}{c}{~}
    \\
    \multicolumn{1}{c}{~} & \multicolumn{3}{c}{$20 \times 20$ TFI model, RGN} \\
    \cline{2-4}
    iteration $200$~ & $-2.51130\bm{56}$
    & $-3.1949\bm{262}$ & $-4.133\bm{5964}$ \\
    iteration $1000$~ & $-2.51130\bm{69}$
    & $-3.1949\bm{974}$ & $-4.133\bm{8354}$ \\
    \cline{2-4}
    \end{tabular}
    \caption{Ground-state energy estimates, normalized by the number of sites.
    Changes between iteration $200$ and iteration $1000$ are marked in bold.}
    \label{tab:results2}
\end{table}

\begin{figure*}[ht!]
    \centering
    \includegraphics[scale=.35]{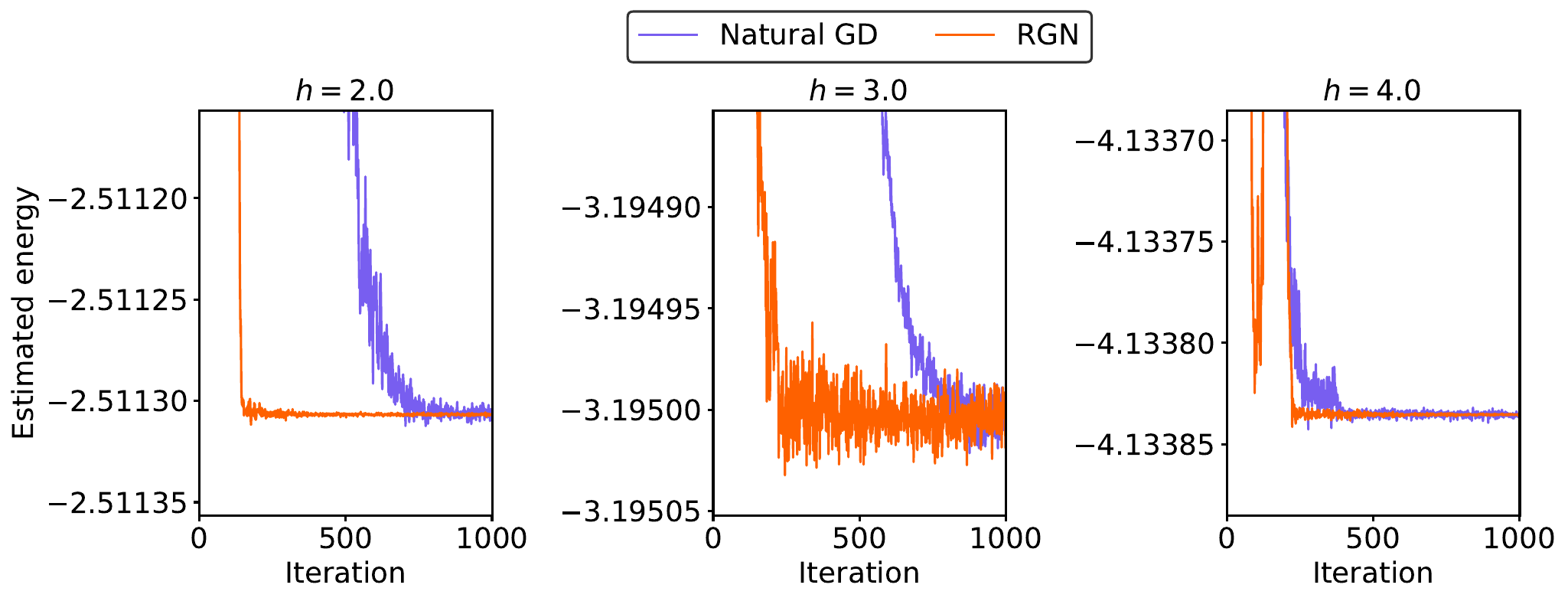}
    \caption{RGN outperforms natural GD in ferromagnetic ($h = 0.5$, left),
    transitional ($h = 1.0$, center),
    and paramagnetic ($h = 1.5$, right) Ising models on a $20 \times 20$ lattice.
    Plot shows ground-state energy estimates normalized by the number of sites.}
    \label{fig:main_result}
\end{figure*}

Summarizing Tables \ref{tab:results1} and \ref{tab:results2}, we find that RGN energies converge more quickly and achieve greater accuracy than natural GD energies.
In 1-D lattices,
RGN is more accurate than natural GD by up to four orders of magnitude, reaching error levels as low as $1.0 \times 10^{-9}$ and $1.6 \times 10^{-9}$.
In 2-D lattices for which exact reference energies are not available,
the energy estimates obtained by RGN are typically lower than those obtained
using natural GD, and the convergence is very fast.
After $200$ iterations, RGN is converged to $4$--$6$ significant digits, whereas natural GD is only converged to $1$--$4$ significant digits.

We further illustrate the comparison between natural GD and RGN
for TFI models
in Figures \ref{fig:main_result} and \ref{fig:main_result_2}.
These figures, showing the complete time series of energy estimates over $1000$ optimization steps, demonstrate that RGN results after $200$ iterations are typically more accurate than natural GD results after $1000$ iterations.
Because RGN is only slightly more expensive than natural GD (less than a factor of two in our experiments),
we conclude that RGN makes it possible to 
obtain accurate ground state estimates with reduced training time and computational cost.

\begin{figure}[h!]
    \centering
    \includegraphics[scale=.35]{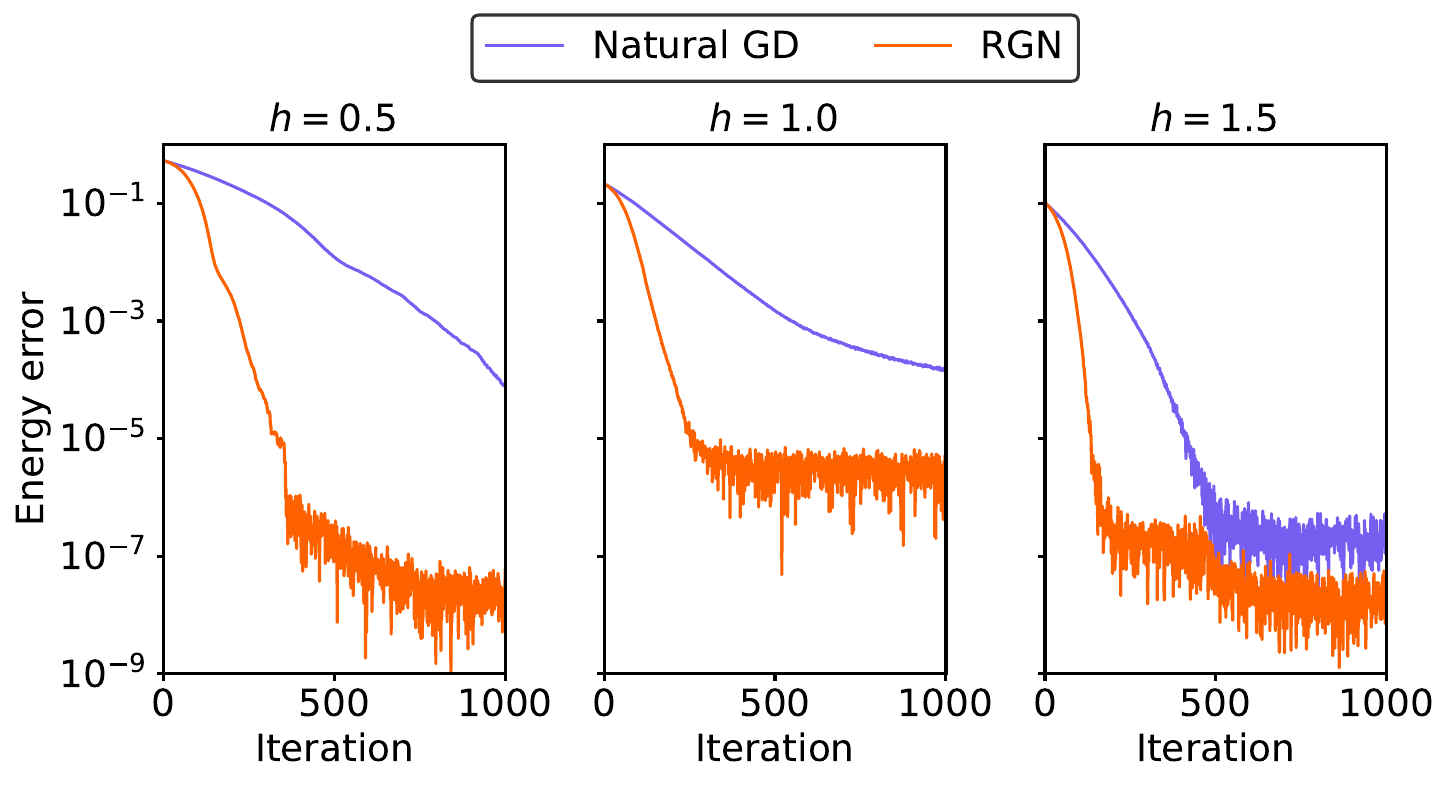}
    \caption{Relative error in ground-state energy estimates for TFI models on a $200 \times 1$ lattice.}
    \label{fig:main_result_2}
\end{figure}



\section{Conclusion}{\label{sec:conclusion}}
This work has analyzed
VMC optimization and sampling methods, leading to both theoretical and computational advancements.
First, we showed that the energy Hessian simplifies dramatically near an eigenstate, depending only on first derivatives of the wavefunction with respect to the parameters.
Taking advantage of this simplification, we introduced a new Rayleigh-Gauss-Newton (RGN) optimizer that can achieve superlinear convergence.
Second, we proved a vanishing-variance property that guarantees VMC energy estimates exhibit reduced variance near an eigenstate.
This principle ensures accuracy in the energies near the ground state but not away from the ground state, so we 
suggested a parallel tempering approach to improve energy and gradient estimation for challenging test problems.

We highlight two opportunities for improving our optimization and sampling methods even further.
First, for very large parametrizations,
the linear system solve in the RGN method becomes numerically challenging. 
To address this difficulty, the Kronecker-factored approximate curvature method for efficient matrix inversion within natural gradient descent \cite{kfac, pfau2020ab}, in addition to the aforementioned matrix-free approach \cite{neuscamman2012optimizing}, could potentially be adapted to RGN.
Second, while parallel tempering is a simple, broadly applicable enhanced sampling method,
there exist a variety of alternative methods \cite{tiwary2016review}.
We anticipate that further analysis of enhanced MCMC sampling
will play an important role in realizing the full potential of VMC in future applications.

\begin{acknowledgments}
RJW and ML would like to acknowledge helpful conversations with Timothy Berkelbach, Aaron Dinner, Sam Greene, Lin Lin, Verena Neufeld, James Smith, Jonathan Siegel,
Erik Thiede, Jonathan Weare, and Huan Zhang.
RJW is supported by New York University's Dean's Dissertation Fellowship and by the National Science Foundation through award DMS-1646339. ML is supported by the National Science Foundation under Award No. 1903031. The authors acknowledge support from the Advanced Scientific
Computing Research Program within the DOE Office of Science through award DE-SC0020427.
Computing resources were provided by New York University's High Performance Computing.
\end{acknowledgments}

\appendix

\section{Proofs}
\label{sec:proofs}

\begin{proof}[Proof of Proposition \ref{prop:rate}]
We prove only the second result. The first result is well-known, and the proof is similar.
Because the wavefunction is analytic at $\bm{\theta}^{\ast}$ and the Wirtinger Hessian
$\big(\begin{smallmatrix}
\bm{H} & \bm{J} \\
\overline{\bm{J}} & \overline{\bm{H}}
\end{smallmatrix}\big)$
is positive definite, as $i \rightarrow \infty$ we find
\begin{equation}
    \bigl(\begin{smallmatrix}
    \bm{g}\left(\bm{\theta}^i\right) \\
    \overline{\bm{g}\left(\bm{\theta}^i\right)}
    \end{smallmatrix}\bigr)
    \sim \bigl(\begin{smallmatrix}
    \bm{H} & \bm{J} \\
    \overline{\bm{J}} & \overline{\bm{H}}
    \end{smallmatrix}\bigr)
    \bigl(\begin{smallmatrix}
    \bm{\theta}^i - \bm{\theta}^{\ast} \\
    \overline{\bm{\theta}^i - \bm{\theta}^{\ast}}
    \end{smallmatrix}\bigr). 
\end{equation}
and
\begin{equation}
    \mathcal{E}\left[\psi_{\bm{\theta}^i}\right]
    - \mathcal{E}\left[\psi_{\bm{\theta}^{\ast}}\right]
    \sim
    \frac{1}{2}
    \Big|
    \bigl(\begin{smallmatrix}
    \bm{H} & \bm{J} \\
    \overline{\bm{J}} & \overline{\bm{H}}
    \end{smallmatrix}\bigr)^{\frac{1}{2}}
    \bigl(\begin{smallmatrix}
    \bm{\theta}^i - \bm{\theta}^{\ast} \\
    \overline{\bm{\theta}^i - \bm{\theta}^{\ast}}
    \end{smallmatrix}\bigr) \Big|^2
\end{equation}
The $i$-th parameter update satisfies
\begin{multline}
    \bigl(
    \begin{smallmatrix}
    \bm{\theta}^{i+1} - \bm{\theta}^{\ast} \\
    \overline{\bm{\theta}^{i+1} - 
    \bm{\theta}^{\ast}}
    \end{smallmatrix}
    \bigr)
    \sim
    \Big[
    \bm{I}
    - \bigl(\begin{smallmatrix}
    \bm{P}_i & \bm{0} \\
    \bm{0} & \overline{\bm{P}_i}
    \end{smallmatrix}\bigr)^{-1}
    \bigl(\begin{smallmatrix}
    \bm{H} & \bm{J} \\
    \overline{\bm{J}} & \overline{\bm{H}}
    \end{smallmatrix}\bigr)\Big]
    \bigl(\begin{smallmatrix}
    \bm{\theta}^i - \bm{\theta}^{\ast} \\
    \overline{\bm{\theta}^i - 
    \bm{\theta}^{\ast}}
    \end{smallmatrix}\bigr) \\
    + \mathcal{O}( | \bm{\theta} - \bm{\theta^{\ast}} |^2 ),
\end{multline}
and the energy ratio satisfies
\begin{multline}
    \frac{\mathcal{E}\left[\psi_{\bm{\theta}^{i+1}}\right]
    - \mathcal{E}\left[\psi_{\bm{\theta}^{\ast}}\right]}{\mathcal{E}\left[\psi_{\bm{\theta}^i}\right]
    - \mathcal{E}\left[\psi_{\bm{\theta}^{\ast}}\right]}
    \sim 
    \\
    \frac{\Big|
    \Big[\bm{I} -
    \bigl(\begin{smallmatrix}
    \bm{H} & \bm{J} \\
    \overline{\bm{J}} & \overline{\bm{H}}
    \end{smallmatrix}\bigr)^{\frac{1}{2}}
    \bigl(\begin{smallmatrix}
    \bm{P}^i & \bm{0} \\
    \bm{0} & \overline{\bm{P}^i}
    \end{smallmatrix}\bigr)^{-1}
    \bigl(\begin{smallmatrix}
    \bm{H} & \bm{J} \\
    \overline{\bm{J}} & \overline{\bm{H}}
    \end{smallmatrix}\bigr)^{\frac{1}{2}}
    \Big]
    \bigl(\begin{smallmatrix}
    \bm{H} & \bm{J} \\
    \overline{\bm{J}} & \overline{\bm{H}}
    \end{smallmatrix}\bigr)^{\frac{1}{2}}
    \big(\begin{smallmatrix}
    \bm{\theta}_i - \bm{\theta}^{\ast} \\
    \overline{\bm{\theta}^i - \bm{\theta}^{\ast}}
    \end{smallmatrix}\big)
    \Big|^2}
    {\Big|
    \bigl(\begin{smallmatrix}
    \bm{H} & \bm{J} \\
    \overline{\bm{J}} & \overline{\bm{H}}
    \end{smallmatrix}\bigr)^{\frac{1}{2}}
    \bigl(\begin{smallmatrix}
    \bm{\theta}_i - \bm{\theta}^{\ast} \\
    \overline{\bm{\theta}^i - 
    \bm{\theta}^{\ast}}
    \end{smallmatrix}\bigr)
    \Big|^2},
\end{multline}
whence
\begin{multline}
    \limsup_{i \rightarrow \infty} \frac{\mathcal{E}\left[\psi_{\bm{\theta}^{i+1}}\right] 
    - \mathcal{E}\left[\psi_{\bm{\theta}^{\ast}}\right]}
    {\mathcal{E}\left[\psi_{\bm{\theta}^{i}}\right] 
    - \mathcal{E}\left[\psi_{\bm{\theta}^{\ast}}\right]} \\
    \leq 
    \limsup_{i \rightarrow \infty} 
    \Big\lVert \bm{I} -
    \bigl(\begin{smallmatrix}
    \bm{H} & \bm{J} \\
    \overline{\bm{J}} & \overline{\bm{H}}
    \end{smallmatrix}\bigr)^{\frac{1}{2}}
    \bigl(\begin{smallmatrix}
    \bm{P}_i & \bm{0} \\
    \bm{0} & \overline{\bm{P}_i}
    \end{smallmatrix}\bigr)^{-1}
    \bigl(\begin{smallmatrix}
    \bm{H} & \bm{J} \\
    \overline{\bm{J}} & \overline{\bm{H}}
    \end{smallmatrix}\bigr)^{\frac{1}{2}}
    \Big\rVert_2^2
\end{multline}
\end{proof}

\begin{proof}[Proof of Proposition \ref{prop:zero_variance}]
    The Markov chain central limit theorem \cite{jones2004markov} guarantees that
    \begin{equation}
        \frac{1}{T} \sum_{t=1}^T
        \begin{pmatrix}
        \overline{\bm{\nu}(\bm{\sigma}_t)} \\
        E_{L}(\bm{\sigma}_t)
        \end{pmatrix} \\
        = \begin{pmatrix}
        \E_{\rho}\!\big[\,\overline{\bm{\nu}(\bm{\sigma})}\,\big] \\
        \mathcal{E}
        \end{pmatrix}
        + \mathcal{O}_p\bigg(\frac{1}{\sqrt{T}}\bigg)
    \end{equation}
    as $T \rightarrow \infty$.
    Next, using the identity
    \begin{multline}
        \big(\bm{x} - \E_{\rho}\!\big[\,\overline{\bm{\nu}(\bm{\sigma})}\,\big]\big)
        \left(y - \mathcal{E}\right) \\
        = 
        \mathcal{O}\big(\big|\bm{x}
        - \E_{\rho}\!\big[\,\overline{\bm{\nu}(\bm{\sigma})}\,\big]\big|^2
        + \left|y - \mathcal{E}\right|^2\big)
    \end{multline}
    and substituting the empirical averages $\bm{x} = \frac{1}{T} \sum_{t=1}^T 
    \overline{\bm{\nu}(\bm{\sigma}_t)}$
    and $y = \frac{1}{T} \sum_{t=1}^T 
    E_{L}(\bm{\sigma}_t)$,
    we obtain
    \begin{align}
        & \hat{\bm{g}}_T \\
        = &\frac{1}{T} \sum_{t=1}^T 
        \overline{\bm{\nu}(\bm{\sigma}_t)}
        E_{L}(\bm{\sigma}_t)
        - \frac{1}{T^2} \sum_{s, t=1}^T 
        \overline{\bm{\nu}(\bm{\sigma}_s)}
        E_{L}(\bm{\sigma}_t) \\
        = &\frac{1}{T}
        \sum_{t=1}^T \bm{g}^{\prime}(\bm{\sigma}_t) + \mathcal{O}_p\Big(\frac{1}{T}\Big).
    \end{align}
    Slutsky's lemma shows that the $\mathcal{O}_p(\frac{1}{T})$ term is asymptotically negligible, and another application of the Markov chain central limit theorem guarantees
    \begin{align}
        & \sqrt{T} \big(\,\hat{\mathcal{E}}_T - \mathcal{E}\big)
        \stackrel{\mathcal{D}}{\rightarrow}
        \mathcal{N}\left(0, v^2\right), \\
        & \sqrt{T}\big(\,\hat{\bm{g}}_T - \bm{g}\big) \stackrel{\mathcal{D}}{\rightarrow} \mathcal{N}\left(\bm{0}, \bm{\Sigma}\right),
    \end{align}
    where the asymptotic variances $v^2$ and $\bm{\Sigma}$ are as given in \eqref{eq:var1} and \eqref{eq:var2}.
\end{proof}

\section{Computations}
\label{sec:computations}

Here, we discuss the computational details for our experiments.
These computations are implemented using the JAX library for Python \cite{jax2018github}, and complete scripts and output are available on github \cite{robert_j_webber_2021_4989655}.
Using these scripts, estimating the ground-state wavefunction for a large lattice is relatively fast, requiring less than four days on a single 48-core CPU node (see Table \ref{tab:runtimes}).
The resulting energies are presented in Table \ref{tab:results1} for TFI models and Table \ref{tab:results2} for XXZ models.

\begin{table}[H]
    \centering
    \begin{tabular}{l|l|l}
        ~ & RGN & Natural GD \\
        \hline
        TFI $200 \times 1$~ &
        $18$--$21$ hrs~ & $12$--$14$ hrs~ \\
        TFI $20 \times 20$~ &
        $58$--$63$ hrs~ & $30$--$32$ hrs~ \\
        XXZ $100 \times 1$~ & $97$--$100$ hrs~ & $85$--$90$ hrs~
    \end{tabular}
    \caption{Runtimes per $1000$ optimization steps on a single $48$-core CPU node, with $2400 \times 20$ MCMC samples per optimization step.}
    \label{tab:runtimes}
\end{table}

We initialize our neural network wavefunction parameters as independent complex-valued $\mathcal{N}\left(0, 0.001\right)$ random variables,
using a random seed of $123$.
We then update our parameters using GD, natural GD, LM, or RGN over $1000$ iterations, as
described in Sections \ref{sub:gradient} and \ref{sub:quasi}.
During the optimizations, 
we increase the penalization parameter $\epsilon$ from $\epsilon = \epsilon_{\textup{min}}$ to $\epsilon = \epsilon_{\textup{max}}$
and increase $\eta$ from $\eta = \eta_{\textup{min}}$ to $\eta = \eta_{\textup{max}}$ at a geometric rate over $\tau$ iterations.
Our specific choices of parameters $\epsilon_{\min}$, $\epsilon_{\max}$, $\eta_{\min}$, $\eta_{\max}$, and $\tau$ are detailed below in Table \ref{tab:penalization}.

\begin{table}[h!]
    \centering
    \begin{tabular}{c|p{.6\linewidth}}
         $\epsilon_{\min}$ ~ & $0.001$  \\
         $\epsilon_{\max}$ ~ & $0.01$ for GD and natural GD,
         $1$ for LM, $1000$ for RGN \\
         $\eta_{\min}$ ~ & $0.001$ \\
         $\eta_{\max}$ ~ & $0.001$ for natural GD, $0.1$ for RGN \\
         $\tau$ ~ & $100$ for deterministic updates, $500$ for stochastic updates
    \end{tabular}
    \caption{Penalization parameters}
    \label{tab:penalization}
\end{table}

Before evaluating the wavefunction $\psi(\bm{\sigma})$ or
wavefunction derivative
$\psi_i(\bm{\sigma})$,
we check whether $\bm{\sigma}$ has `mostly negative' magnetization, as defined by
\begin{equation}
\label{eq:mostly_neg}
    2 \sum_i \bm{\sigma}_i + \bm{\sigma}_1 < 0.
\end{equation}
If we encounter a configuration $\bm{\sigma}$ 
for which \eqref{eq:mostly_neg} is not satisfied, we transform it to $-\bm{\sigma}$.
Indeed, the mostly negative configurations suffice to determine the complete wavefunction given the symmetry condition
$\psi(\bm{\sigma}) = \psi\left(-\bm{\sigma}\right)$,
and VMC can lead to low-quality ground-state wavefunction estimates when this symmetry condition is not enforced.

For 1-D and 2-D lattices, we can evaluate the log wavefunction and its derivatives in $\mathcal{O}\left(\alpha n \log n\right)$ operations using
the discrete Fourier transform $\mathcal{F}$ and its inverse $\mathcal{F}^{-1}$.
To show this, we write
\begin{equation}
\label{eq:dft1}
    \log \psi_{\bm{w},\bm{b}} (\bm{\sigma}) = \sum_{i=1}^{\alpha} \sum_j \log \cosh \bm{\theta}_{ij},
\end{equation}
where we have introduced angles
\begin{equation}
\label{eq:dft2}
    \bm{\theta}_{ij}= \left(\mathcal{F}^{-1} \left( \mathcal{F} \bm{w}_{i\cdot} \odot \overline{\mathcal{F} \bm{\sigma}} \right)\right)_j + \bm{b}_i
\end{equation}
and we have used $\odot$ to represent element-wise multiplication.
Similarly, we write
\begin{align}
    & \frac{\partial \log \psi_{\bm{w},\bm{b}}}{\partial \bm{b}_i} (\bm{\sigma})
    = \sum_j \tanh \bm{\theta}_{ij}, \\
    & \frac{\partial \log \psi_{\bm{w},\bm{b}}}{\partial \bm{w}_{ij}} (\bm{\sigma})
    = \left(\mathcal{F}^{-1}\left(\mathcal{F} \left(\tanh \bm{\theta}_{i \cdot}\right)
    \odot \mathcal{F} \bm{\sigma} \right)\right)_j.
\end{align}

When optimizing VMC wavefunctions,
we occasionally encounter a sudden increase in the norm of the parameter updates, here defined as a factor of two or greater.
When such a large update occurs,
in addition to immediately restricting the size of the parameter update (by decreasing $\epsilon$),
we restore $\epsilon = \epsilon_{\textup{min}}$ and $\eta = \eta_{\textup{min}}$
and restart the geometric progression.

Lastly, to obtain the energy estimates reported in Tables \ref{tab:results1} and \ref{tab:results2},
we run the MCMC chains for an additional $2000 \times n$ time steps
and evaluate the local energies at intervals of $n$ time steps.

\bibliography{references}

\end{document}